\documentclass{article}
\usepackage{nips12submit_e,times}
\usepackage[numbers]{natbib}
\usepackage{graphicx}
\usepackage{amsmath}
\usepackage{amssymb}
\usepackage{amsthm}
\usepackage{float}
\usepackage{wrapfig}
\usepackage{listings}
\usepackage[font=small,labelfont=bf,tableposition=top]{caption}
\usepackage[font=footnotesize]{subcaption}
\usepackage{algorithm}
\usepackage{algorithmic}

\DeclareGraphicsExtensions{.pdf}
\DeclareMathOperator*{\argmin}{arg\,min}

\floatstyle{ruled}
\newfloat{algorithm}{tbp}{loa}
\providecommand{\algorithmname}{Algorithm}
\floatname{algorithm}{\protect\algorithmname}

\theoremstyle{plain}
\newtheorem*{prop*}{\protect\propositionname}
\newtheorem{prop}{\protect\propositionname}

\title{A Convex Formulation for Learning Scale-Free Networks via Submodular Relaxation}

\author{
Aaron J. Defazio\\
NICTA/Australian National University\\
Canberra, ACT, Australia\\
\texttt{aaron.defazio@anu.edu.au}\\
\And
Tiberio S. Caetano\\
NICTA/ANU/University of Sydney\\
Canberra and Sydney, Australia\\
\texttt{tiberio.caetano@nicta.com.au}\\
}

\providecommand{\propositionname}{Proposition}

\nipsfinalcopy 

\begin{document}

\maketitle

\vspace{-2mm}

\begin{abstract}
	A key problem in statistics and machine learning is the determination of network structure from data. We consider the case where the structure of the graph to be reconstructed is known to be scale-free. We show that in such cases it is natural to formulate structured sparsity inducing priors using submodular functions, and we use their Lov\'asz extension to obtain a convex relaxation. For tractable classes such as Gaussian graphical models, this leads to a convex optimization problem that can be efficiently solved. We show that our method results in an improvement in the accuracy of reconstructed networks for synthetic data. We also show how our prior encourages scale-free reconstructions on a bioinfomatics dataset.
\end{abstract}


\section*{Introduction}
\label{sec:introduction}
Structure learning for graphical models is a problem that arises in many contexts. In applied statistics, undirected graphical models can be used as a tool for understanding the underlying conditional independence relations between variables in a dataset. For example, in bioinfomatics Gaussian graphical models are fitted to data resulting from micro-array experiments, where the fitted graph can be interpreted as a gene expression network \cite{gene-expression-2004}.

In the context of Gaussian models, the structure learning problem is known as covariance selection \citep{dempster}. The most common approach is the application of sparsity inducing regularization to the maximum likelihood objective. There is a significant body of literature, more than 30 papers by our count, on various methods of optimizing the $L_1$ regularized covariance selection objective alone (see the recent review by \citet{katya}).

Recent research has seen the development of \emph{structured} sparsity, where more complex prior knowledge about a sparsity pattern can be encoded. Examples include group sparsity \cite{group-sparsity}, where parameters are linked so that they are regularized in groups. More complex sparsity patterns, such as region shape constraints in the case of pixels in an image \cite{sparse-pca}, or hierarchical constraints \cite{hierarchical-sparse} have also been explored.

In this paper, we study the problem of recovering the structure of a Gaussian graphical model under the assumption that the graph recovered should be scale-free. Many real-world networks are known a priori to be scale-free and therefore enforcing that knowledge through a prior seems a natural idea. Recent work has offered an approach to deal with this problem which results in a non-convex formulation \cite{reweighted-l1}. Here we present a convex formulation. We show that scale-free networks can be induced by enforcing submodular priors on the network's degree distribution, and then using their convex envelope (the Lov\'asz extension) as a convex relaxation \cite{subbach}. 
The resulting relaxed prior has an interesting non-differentiable structure, which poses challenges to optimization. We outline a few options for solving the optimisation problem via proximal operators \cite{sparsity-opt}, in particular an efficient dual decomposition method. Experiments on both synthetic data produced by scale-free network models and a real bioinformatics dataset suggest that the convex relaxation is not weak: we can infer scale-free networks with similar or superior accuracy than in \cite{reweighted-l1}.


\section{Combinatorial Objective}
\label{sec:objective}
Consider an undirected graph with edge set $E$ and node set $V$, where $n$ is the number of nodes. We denote the degree of node $v$ as $d_{E}(v)$, and the complete graph with $n$ nodes as $K_n$. We are concerned with placing priors on the degree distributions of graphs such as $(V,E)$. By degree distribution, we mean the bag of degrees $\left\{ d_{E}(v) | v \in V \right\}$.

A natural prior on degree distributions can be formed from the family of exponential random graphs \cite{exp-graphs}. Exponential random graph (ERG) models assign a probability to each $n$ node graph using an exponential family model. The probability of each graph depends on a small set of sufficient statistics, in our case we only consider the degree statistics. A ERG distribution with degree parametrization takes the form:
\begin{gather}
	p(G=(V,E);h) 
				\approx \frac{1}{Z(h)} 
				\exp{\left [\, -\sum_{v\in V} h(d_{E}(v)) \, \right ]},
\end{gather}

The degree weighting function $h:\mathbb{Z}^{+}\rightarrow\mathbb{R}$
encodes the preference for each particular degree. The function $Z$ is chosen so that the distribution is correctly normalized over $n$ node graphs.

A number of choices for $h$ are reasonable; A geometric series $h(i)\propto 1 - \alpha^{i}$ with $\alpha \in (0,1)$ has been proposed by \citet{geo-erg} and has been widely adopted. However for encouraging scale free graphs we require a more rapidly increasing sequence. It is instructive to observe that, under the strong assumption that each node's degree is independent of the rest, $h$ grows logarithmically. To see this, take a scale free model with scale $\alpha$; the joint distribution takes the form:
\begin{gather*}
	p(G=(V,E);\epsilon, \alpha) 
				\approx \frac{1}{Z(\epsilon, \alpha)} 
				\prod_{v \in V} (d_{E}(v)+\epsilon)^{-\alpha},
\end{gather*}
where $\epsilon>0$ is added to prevent infinite weights. Putting this into ERG form gives the weight sequence $h(i)=\alpha \log(i + \epsilon)$. We will consider this and other functions $h$ in Section \ref{sec:degree-priors}. We intend to perform maximum a posteriori (MAP) estimation of a graph structure using such a distribution as a prior, so the object of our attention is the negative log-posterior, which we denote $F$:
\begin{align}
	F(E) = \sum_{v \in V} h(d_{E}(v)) + \text{const.}
\end{align}
So far we have defined a function on edge sets only, however in practice we want to optimize over a weighted graph, which is intractable when using discontinuous functions such as $F$. We now consider the properties of $h$ that lead to a convex relaxation of $F$.

\section{Submodularity}
\label{sec:submodularity}
A set function $F:\,2^{E}\rightarrow\mathbb{R}$ on $E$ is a non-decreasing submodular function if for all $A\subset B\subset E$ and $x\in E\backslash B$ the following conditions hold:
\begin{align}
 F(A\cup\{x\})-F(A) & \ge F(B\cup\{x\})-F(B) \tag{submodularity} \\
 \text{and } F(A) & \leq  F(B). \tag{non-decreasing}
\end{align}
The first condition can be interpreted as a diminishing returns condition;
adding $x$ to a set $A$ increases $F(A)$ by more than adding it
to a larger set $B$, if $B$ contains $A$.

We now consider a set of conditions that can be placed on $h$ so that $F$ is submodular.
\begin{prop}
\label{prop:h}
Denote $h$ as tractable if $h$ is non-decreasing, concave and $h(0)=0$.
For tractable $h$, $F$ is a non-decreasing submodular function.\end{prop}
\begin{proof}
First note that the degree function is a set cardinality function,
and hence modular. A concave transformation of a modular function
is submodular \citep{bach-tutorial}, and the sum of submodular functions
is submodular. 
\end{proof} 
The concavity restriction we impose on $h$ is the key ingredient that allows us to use submodularity to enforce a prior for scale-free networks; any prior favouring long tailed degree distributions must place a lower weight on new edges joining highly connected nodes than on those joining other nodes. As far as we are aware, this is a novel way of mathematically modelling the `preferential attachment' rule  \cite{ba-model} that gives rise to scale-free networks: through non-decreasing submodular functions on the degree distribution.

Let $X$ denote a symmetric matrix of edge weights. A natural convex relaxation of $F$ would be the convex envelope of $F(\textrm{Supp}(X))$ under some restricted domain. For tractable $h$, we have by construction that $F$ satisfies the conditions of Proposition 1 in \cite{subbach}, so that the convex envelope of $F(\textrm{Supp}(X))$ on the $L_{\infty}$ ball is precisely the Lov\'asz extension evaluated on $|X|$. The Lov\'asz extension for our function is easy to determine as it is a sum of ``functions of cardinality'' which are considered in \cite{subbach}. Below is the result from \cite{subbach} adapted to our problem.
\begin{prop}
\label{prop:convex-envelope}
Let $X_{i,(j)}$ be the weight of the $j$th edge connected to $i$, under a decreasing ordering by absolute value (i.e $|X_{i,(0)}|\geq|X_{i,(1)}|\geq...\geq|X_{i,(n-1)}|$). The notation $(i)$ maps from sorted order to the natural ordering, with the diagonal not included. Then the convex envelope of $F$ for tractable $h$ over the $L_{\infty}$ norm unit ball is:
\begin{gather*}
\Omega(X) = \sum_{i=0}^{n}\,\sum_{k=0}^{n-1}\left(h(k+1)-h(k)\right)\lvert X_{i,(k)} \rvert.
\end{gather*}
This function is piece-wise linear and convex.
\end{prop}
The form of $\Omega$ is quite intuitive. It behaves like a $L_1$
norm with an additional weight on each edge that depends on how the
edge ranks with respect to the other edges of its neighbouring nodes.

\section{Optimization}
\label{sec:optimization}
We are interested in using $\Omega$ as a prior, for optimizations of the form
\[
\textrm{minimize}_{X}\quad f(X)=g(X)+\alpha\Omega(X),
\]
for convex functions $g$ and prior strength parameters $\alpha\in\mathbb{R}^{+}$,
over symmetric $X$. We will focus on the simplest structure learning problem that occurs in graphical model training, that of Gaussian models. In which case we have
\[
g(X)=\left\langle X,C\right\rangle -\log\det X,
\]
where $C$ is the observed covariance matrix of our data. The support of $X$ will then be the set of edges in the undirected graphical model together with the node precisions. This function is a rescaling of the maximum likelihood objective. In order for the resulting $X$ to define a normalizable distribution, $X$ must be restricted to the cone of positive definite matrices. This is not a problem in practice as $g(X)$ is infinite on the boundary of the PSD cone, and hence the constraint can be handled by restricting optimization steps to the interior of the cone. In fact $X$ can be shown to be in a strictly smaller cone, $X^{*}\succeq aI$, for $a$ derivable from $C$ \citep{smooth-covsel}. This restricted domain is useful as $g(X)$ has Lipschitz continuous gradients over $X\succeq aI$ but not over all positive definite matrices \citep{cov-admm}.

There are a number of possible algorithms that can be applied for optimizing a convex non-differentiable objective such as $f$. \citet{subbach} suggests two approaches to optimizing functions involving submodular relaxation priors; a subgradient approach and a proximal approach.

Subgradient methods are the simplest class of methods for optimizing non-smooth convex functions. They provide a good baseline for comparison with other methods. For our objective, a subgradient is simple to evaluate at any point, due to the piecewise continuous nature of $\Omega(X)$. Unfortunately (primal) subgradient methods for our problem will not return sparse solutions except in the limit of convergence. They will instead give intermediate values that oscillate around their limiting values.

An alternative is the use of proximal methods \cite{subbach}. Proximal methods exhibit superior convergence in comparison to subgradient methods, and produce sparse solutions. Proximal methods rely on solving a simpler optimization problem, known as the \textit{proximal} operator at each iteration:
\[
\argmin_{X}\, \left[ \alpha\Omega(X) + \frac{1}{2} \left\Vert X-Z\right\Vert _{2}^{2} \right],
\]
where $Z$ is a variable that varies at each iteration. For many problems of interest, the proximal operator can be evaluated using a closed form solution. For non-decreasing submodular relaxations, the proximal operator can be evaluated by solving a submodular minimization on a related (not necessarily non-decreasing) submodular function \citep{subbach}.

\citet{subbach} considers several example problems where the proximal operator can be evaluated using fast graph cut methods. For the class of functions we consider, graph-cut methods are not applicable. Generic submodular minimization algorithms could be as slow as $O(n^{12})$ for a $n$-vertex graph, which is clearly impractical \cite{submod-book}. We will instead propose a dual decomposition method for solving this proximal operator problem in Section \ref{sec:proximal}. 

For solving our optimisation problem, instead of using the standard proximal method (sometimes known as ISTA), which involves a gradient step followed by the proximal operator, we propose to use the alternating direction method of multipliers (ADMM), which has shown good results when applied to the standard $L_1$ regularized covariance selection problem \cite{cov-admm}. Next we show how to apply ADMM to our problem.

\subsection{Alternating direction method of multipliers}
\label{sec:admm}
The alternating direction method of multipliers (ADMM, \citet{admm})
is one approach to optimizing our objective that has a number of advantages
over the basic proximal method. Let $U$ be the matrix of dual variables for the decoupled problem:
\begin{gather*}
\text{minimize}_{X} \quad g(X)+\alpha\Omega(Y), \\
s.t.\quad X=Y.
\end{gather*}
Following the presentation of the algorithm in \citet{admm}, given the values $Y^{(l)}$ and $U^{(l)}$ from iteration $l$, with $U^{(0)}=0_n$ and $Y^{(0)}=I_n$ the ADMM updates for iteration $l+1$ are:
\begin{align*}
X^{(l+1)} & = \argmin_{X}\left[\left\langle X,C\right\rangle -\log\det X
			+ \frac{\rho}{2}||X-Y^{(l)}+U^{(l)} ||_{2}^{2}\right] \\
Y^{(l+1)} & = \argmin_{Y}\left[\alpha\Omega(Y) 
			+ \frac{\rho}{2}||X^{(l+1)}-Y+U^{(l)} ||_{2}^{2}\right] \\
U^{(l+1)} & = U^{(l)}+X^{(l+1)}-Y^{(l+1)},
\end{align*}
where $\rho>0$ is a fixed step-size parameter (we used $\rho=0.5$). The advantage of this form is that both the $X$ and $Y$ updates are a proximal operation. It turns out that the proximal operator for $g$ (i.e. the $X^{(l+1)}$ update)  actually has a simple solution \citep{cov-admm} that can be computed by taking an eigenvalue decomposition
$Q^{T}\Lambda Q=\rho(Y-U)-C$, where $\Lambda=\textrm{diag}(\lambda_{1},\dots,\lambda_{n})$
and updating the eigenvalues using the formula
\[
\lambda_{i}^{\prime}:=\frac{\lambda_{i}+\sqrt{\lambda_{i}^{2}+4\rho}}{2\rho}
\]
to give $X=Q^{T}\Lambda^{\prime}Q$. The stopping criterion we used was $||X^{(l+1)}-Y^{(l+1)}||<\epsilon$ and $||Y^{(l+1)}-Y^{(l)}||<\epsilon$. In practice the ADMM method is one of the fastest methods for $L_1$ regularized covariance selection. \citet{cov-admm} show that convergence is guaranteed if additional cone restrictions are placed on the minimization with respect to $X$, and small enough step sizes are used. For our degree prior regularizer, the difficultly is in computing the proximal operator for $\Omega$, as the rest of the algorithm is identical to that presented in \citet{admm}. We now show how we solve the problem of computing the proximal operator for $\Omega$.

\subsection{Proximal operator using dual decomposition}
\label{sec:proximal}
Here we describe the optimisation algorithm that we effectively use for computing the proximal operator. The regularizer $\Omega$ has a quite complicated structure due to
the interplay between the terms involving the two end points for each
edge. We can decouple these terms using the dual decomposition technique,
by writing the proximal operation for a given $Z=Y-U$ as:
\begin{gather*}
\textrm{minimize}_{X}=\frac{\alpha}{\rho}\sum_{i}^{n}\,\sum_{k}^{n-1}\left(h(k+1)-h(k)\right)\left|X_{i,(k)}\right|+\frac{1}{2}||X-Z||_{2}^{2} \\
s.t.\quad X=X^{T}.
\end{gather*}
The only difference so far is that we have made the symmetry constraint
explicit. Taking the dual gives a formulation where the upper and
lower triangle are treated as separate variables. The dual variable
matrix $V$ corresponds to the Lagrange multipliers of the symmetry
constraint, which for notational convenience we store in an anti-symmetric
matrix. The dual decomposition method is given in Algorithm \ref{alg:dd-main}.
\begin{algorithm}[!h]
    	\small
      \centering
\begin{algorithmic}
	\STATE {\bfseries input:} matrix $Z$, constants $\alpha$, $\rho$
	\STATE {\bfseries input:} step-size $0<\eta<1$
	\STATE {\bfseries initialize:} $X=Z$
	\STATE {\bfseries initialize:} $V=0_n$
	\REPEAT
		\FOR{$l=0$ {\bfseries until} $n-1$}
			\STATE $X_{l*} = \text{solveSubproblem}(Z_{l*}, V_{l*})$ \emph{\# Algorithm  \ref{alg:dd-subproblem}}
		\ENDFOR
		\STATE $V = V + \eta(X - X^T)$
	\UNTIL {$||X-X^T|| < 10^{-6}$}
	\STATE $X = \frac{1}{2}(X+X^T)$ \emph{\# symmetrize}
	\STATE {\bfseries round:} any $|X_{ij}| < 10^{-15}$ to $0$
	\STATE {\bfseries return} X
\end{algorithmic}
\caption{\label{alg:dd-main}Dual decomposition main}
\end{algorithm}

We use the notation $X_{i*}$ to denote the $i$th row of $X$. Since this is a dual method, the primal variables $X$ are not feasible (i.e. symmetric) until convergence.
Essentially we have decomposed the original problem, so that now we
only need to solve the proximal operation for each node in isolation, namely the subproblems:
\begin{gather}
\label{eq:dd}
\forall i.\; X^{(l+1)}_{i*}=\arg\min_{x}\,
\frac{\alpha}{\rho}\sum_{k}^{n-1}\left(h(k+1)-h(k)\right)\left|x_{(k)}\right| 
+ ||x-Z_{i*}+V_{i*}^{(l)}||_{2}^{2}.
\end{gather}
Note that the dual variable has been integrated into the quadratic
term by completing the square. As the diagonal elements of $X$ are
not included in the sort ordering, they will be minimized by $X_{ii}=Z_{ii},$ for all $i$. Each subproblem is strongly convex as they consist of convex terms plus a positive quadratic term. This implies that the dual problem is differentiable (as the subdifferential contains only one subgradient), hence the $V$ update is actually gradient ascent. Since a fixed step size is used, and the dual is Lipschitz continuous,
for sufficiently small step-size convergence is guaranteed. In practice we used $\eta=0.9$ for all our tests.

This dual decomposition subproblem can also be interpreted as just
a step within the ADMM framework. If applied in a standard way, only
one dual variable update would be performed before another expensive
eigenvalue decomposition step. Since each iteration of the dual decomposition
is much faster than the eigenvalue decomposition, it makes more sense
to treat it as a separate problem as we propose here. It also ensures
that the eigenvalue decomposition is only performed on symmetric matrices.

Each subproblem in our decomposition is still a non-trivial problem. They do have a closed form solution, involving a sort and several passes over the node's edges, as described in Algorithm \ref{alg:dd-subproblem}.
\vspace{2mm}
\begin{prop}
Algorithm \ref{alg:dd-subproblem} solves the subproblem in equation \ref{eq:dd}.
\end{prop}
\vspace{-0.75em}
\emph{Proof:} See Appendix 1 in the supplementary material. The main subtlety is the grouping together of elements induced at the non-differentiable points. If multiple edges connected to the same node have the same absolute value, their subdifferential becomes the same, and they behave as a single point whose weight is the average. To handle this grouping, we use a disjoint-set data-structure, where each $x_{j}$ is either in a singleton set, or grouped in a set with other elements, whose absolute value is the same.

\begin{algorithm}[t]
    	\small
      \centering
\begin{algorithmic}
	\STATE {\bfseries input:} vectors $z$, $v$
	\STATE {\bfseries initialize:} Disjoint-set datastructure 
		with set membership function $\gamma$
	\STATE $w = z-v\quad$ \emph{\# $w$ gives the sort order}
	\STATE $u = 0_n$
	\STATE {\bfseries build:} sorted-to-original position function $\mu$
			under descending absolute value order of $w$, excluding the diagonal
	\FOR{$k=0$ {\bfseries until} $n-1$}
		\STATE $j = \mu(k)$
		\STATE $u_j = |w_j| -  \frac{\alpha}{\rho} \left(h(k+1)-h(k)\right)$
		\STATE $\gamma(j).\textrm{value} = u_j$
		\STATE r = k
		\WHILE{$r>1$ and $ \gamma(\mu(r)).\textrm{value}  \geq 
			\gamma(\mu(r-1)).\textrm{value} $ }
			\STATE {\bfseries join:} the sets containing $\mu(r)$ and $\mu(r-1)$
			\STATE $\gamma(\mu(r)).\textrm{value} = 
					\frac{1}{|\gamma(\mu(r))|} 
					\sum_{i \in \gamma(\mu(r))} u_i $
			\STATE {\bfseries set:} $r$ to the first element 
					of $\gamma(\mu(r))$ by the sort ordering
		\ENDWHILE
	\ENDFOR
	\FOR{$i=1$ {\bfseries to} $N$}
		\STATE $x_i = \gamma(i).\textrm{value}$
		\IF{ $x_i < 0$ }
			\STATE $x_j = 0\quad $ \emph{\# negative values imply shrinkage to 0}
		\ENDIF
		\IF{ $w_i < 0$ }
			\STATE $x_j = -x_j\quad$ \emph{\# Correct orthant}
		\ENDIF
	\ENDFOR
	\STATE {\bfseries return} $x$
\end{algorithmic}
\caption{\label{alg:dd-subproblem}Dual decomposition subproblem (\emph{solveSubproblem})}
\end{algorithm}
  
\section{Alternative degree priors}
\label{sec:degree-priors}
Under the restrictions on $h$ detailed in Proposition \ref{prop:h}, several other choices seem reasonable. The scale free prior can be smoothed somewhat, by the addition of a linear term, giving
\[
h_{\epsilon,\beta}(i)=\log(i+\epsilon)+\beta i,
\]
where $\beta$ controls the strength of the smoothing. A slower diminishing choice would be a square-root function such as
\[
h_{\beta}(i)=(i+1)^{\frac{1}{2}}-1+\beta i.
\]
This requires the linear term in order to correspond to a normalizable prior.

Ideally we would choose $h$ so that the expected degree distribution under the ERG model matches the particular form we wish to encourage. Finding such a $h$ for a particular graph size and degree distribution amounts to maximum likelihood parameter learning, which for ERG models is a hard learning problem. The most common approach is to use sampling based inference. Approaches based on Markov chain Monte Carlo techniques have been applied widely to ERG models \citep{erg-sampling} and are therefore applicable to our model.

\section{Related Work}
\label{sec:related-work}
The covariance selection problem has recently been addressed by \citet{reweighted-l1} using reweighted $L_{1}$ regularization. They minimize the following objective: 
\[
f(X)=\left\langle X,C\right\rangle 
-\log\det X+\alpha\sum_{v \in V}\log\left(\left\Vert X_{\neg v}\right\Vert 
	+\epsilon\right)
+\beta\sum_{v}\left| X_{vv}\right|.
\]
The regularizer is split into an off diagonal term which is designed
to encourage sparsity in the edge parameters, and a more traditional diagonal term. Essentially they use $\left\Vert X_{\neg v}\right\Vert $ as the continuous counterpart of node $v$'s degree. The biggest difficulty with this objective is the log term, which makes $f$ highly non-convex. This can be contrasted to our approach, where we start with essentially the same combinatorial prior, but we use an alternative, convex relaxation.

The reweighted $L_{1}$ \cite{reweighted-l1-boyd} aspect refers to the method of optimization applied. A double loop method is used, in the same class as EM methods and difference of convex programming, where each $L_{1}$ inner problem gives a monotonically improving lower bound on the true solution.

\section{Experiments}
\label{sec:results}
\vspace{-3mm}
\textbf{Reconstruction of synthetic networks.} We performed a comparison against the reweighted $L_{1}$ method of \citet{reweighted-l1},
and a standard $L_{1}$ regularized method, both implemented using
ADMM for optimization. Although \citet{reweighted-l1} use the glasso \citep{glasso} method for the inner loop, ADMM will give identical results, and is usually faster \cite{cov-admm}. Graphs with $60$ nodes were generated using
both the Barabasi-Albert model \citep{ba-model} and a predefined
degree distribution model sampled using the method from \citet{random-degree-graph}
implemented in the NetworkX software package.
Both methods generate scale-free graphs; the BA model exhibits a scale
parameter of 3.0, whereas we fixed the scale parameter at 2.0 for
the other model. 
To define a valid Gaussian model, edge weights of $X_{ij}=-0.2$ were assigned, and the node weights were set at $X_{ii} = 0.5-\sum_{i\neq j} X_{ij}$  so as to make the resulting precision matrix diagonally dominant. The resulting Gaussian graphical model was sampled 500 times. The covariance matrix of these samples was formed, then normalized to have diagonal uniformly 1.0. We tested with the two $h$ sequences described in section \ref{sec:degree-priors}. The parameters for the degree weight sequences were chosen by grid search on random instances separate from those we tested on. The resulting ROC curves for the Hamming reconstruction loss are shown in Figure \ref{fig:synth-roc}. Results were averaged over 30 randomly generated graphs for each each figure.

We can see from the plots that our method with the square-root weighting presents results superior to those from \citet{reweighted-l1} for these datasets. This is encouraging particularly since our formulation is convex while the one from \citet{reweighted-l1} isn't. Interestingly, the log based weights give very similar but not identical results to the reweighting scheme which also uses a log term. The only case where it gives inferior reconstructions is when it is forced to give a sparser reconstruction than the original graph.
\begin{figure}[t]
	\center
	\includegraphics[scale=0.45]{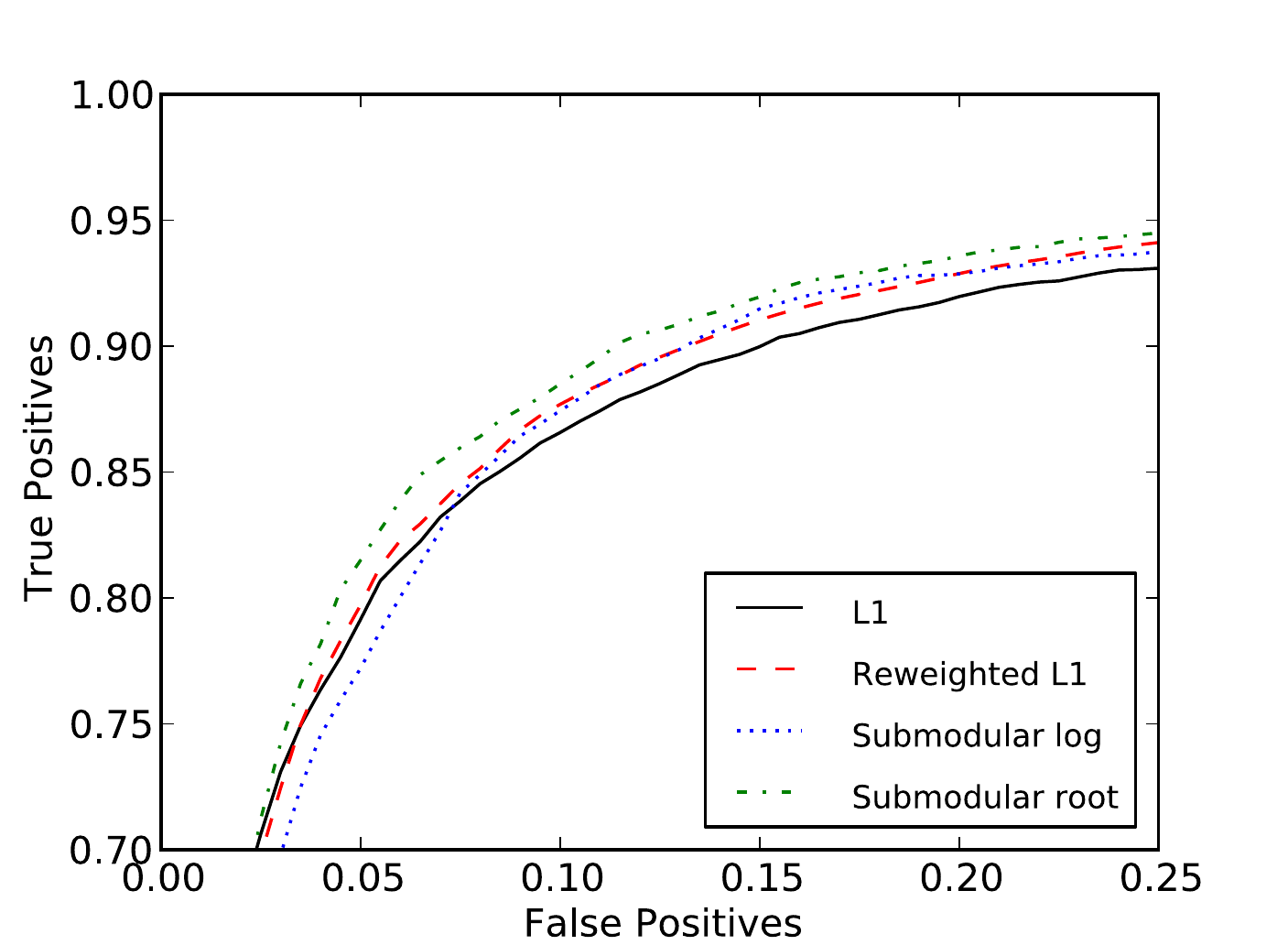}
	\hfill
	\includegraphics[scale=0.45]{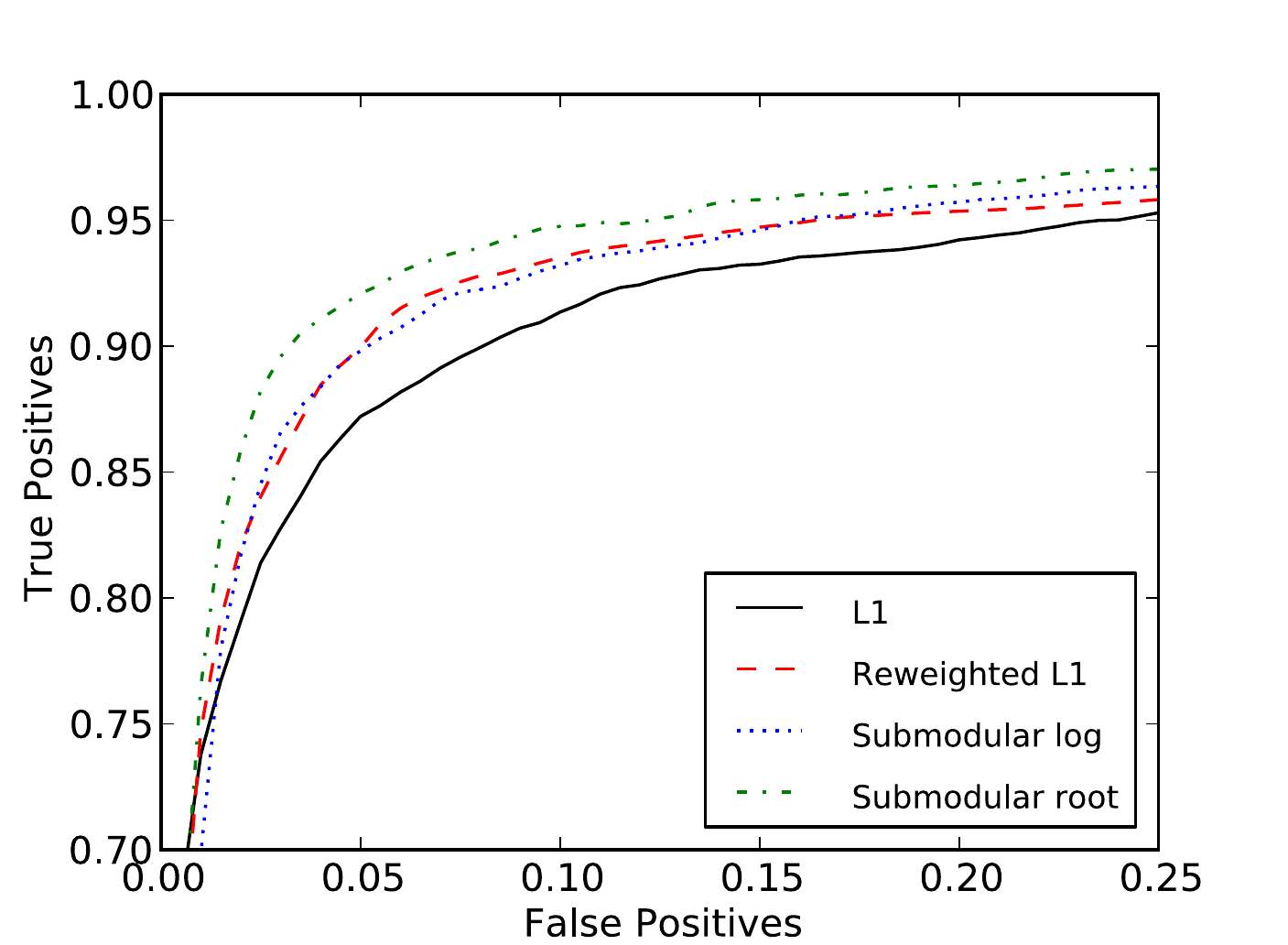}
	\caption{\small \label{fig:synth-roc}ROC curves for BA model (left)
			 and fixed degree distribution model (right)}
\end{figure}
\begin{figure}[t]
	\includegraphics[scale=0.20]{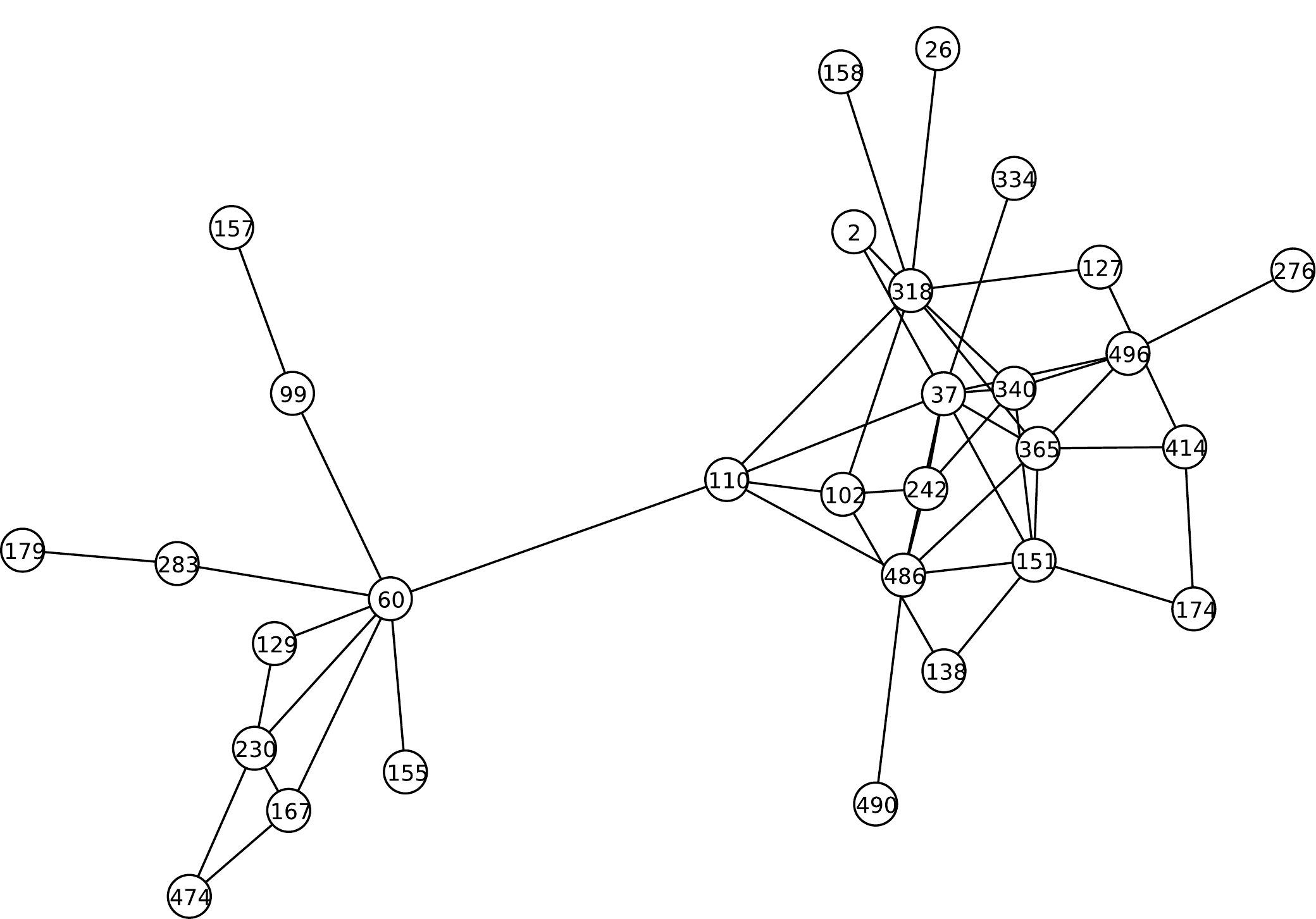}
	\hfill
	\includegraphics[scale=0.20]{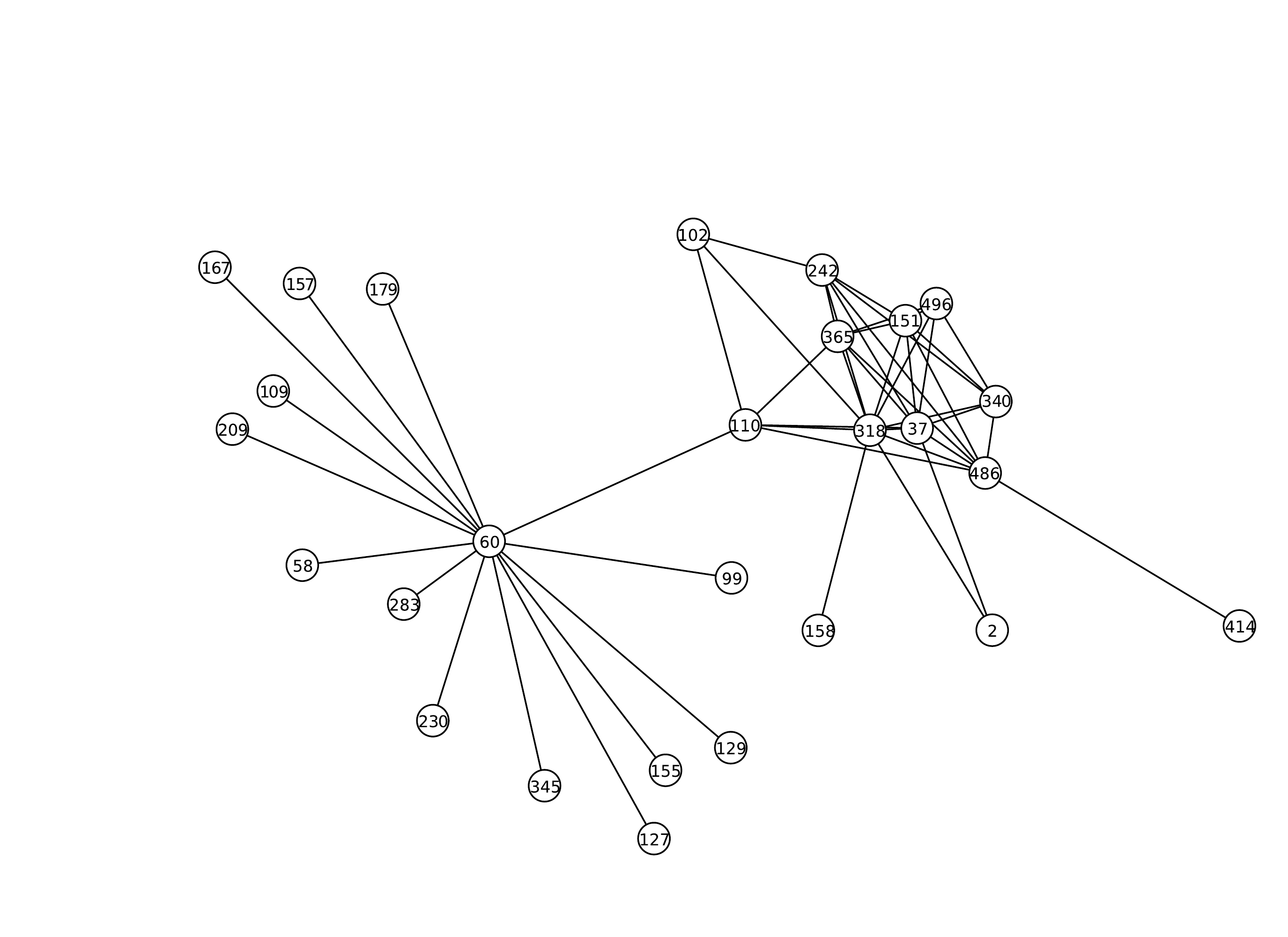}
	\hfill
	\includegraphics[scale=0.20]{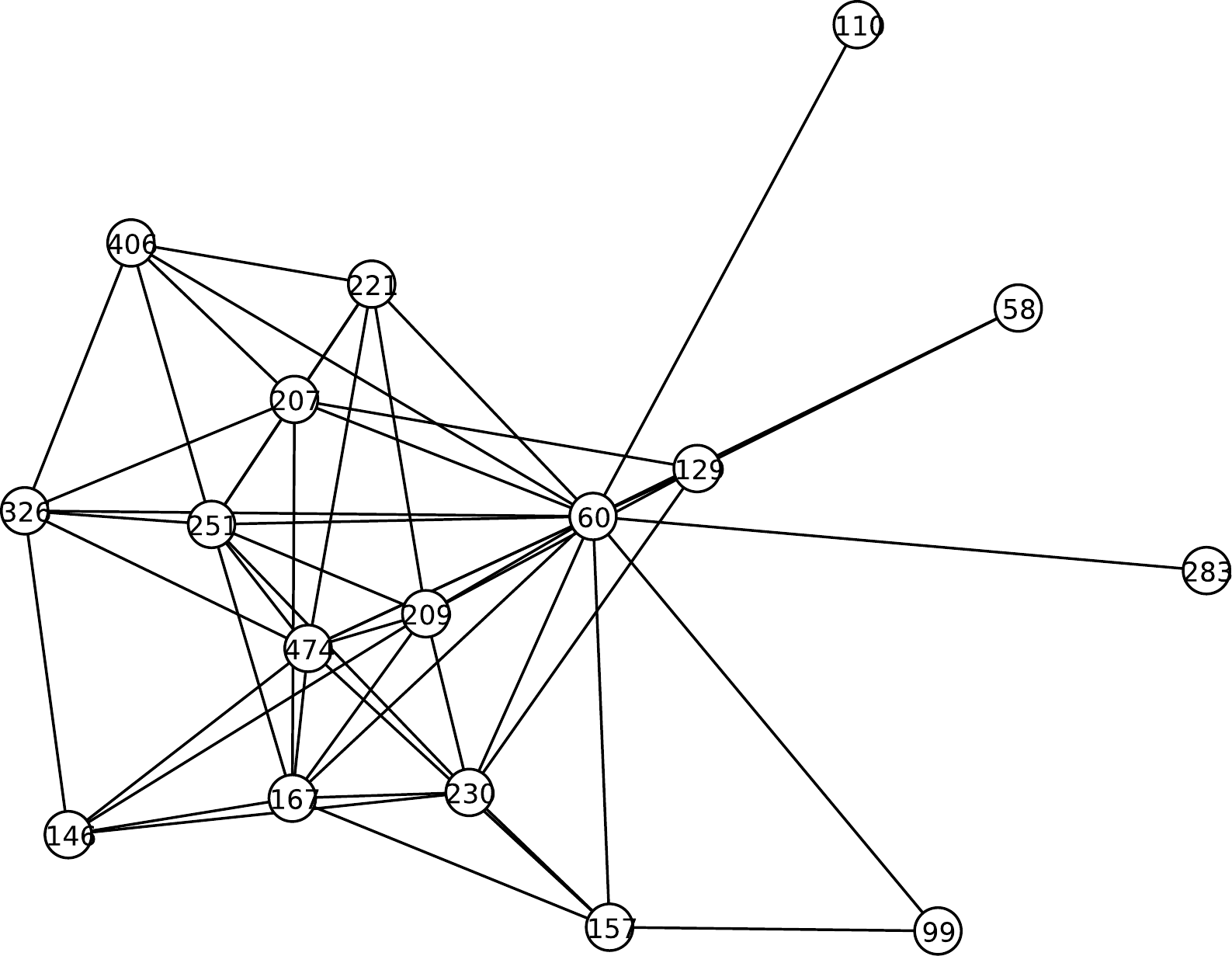}
	\caption{\small \label{fig:gene-recon} Reconstruction of a gene association network using 
			$L_1$ (left), submodular relaxation (middle), and reweighted $L_1$ (right) methods}
\end{figure}

\textbf{Reconstruction of a gene activation network.} A common application of sparse covariance selection is the estimation of gene association networks from experimental data. A covariance matrix of gene co-activations from a number of independent micro-array experiments is typically formed, on which a number of methods, including sparse covariance selection, can be applied. Sparse estimation is key for a consistent reconstruction due to the small number of experiments performed. Many biological networks are conjectured to be scale-free, and additionally ERG modelling techniques are known to produce good results on biological networks \cite{erg-gene}. So we consider micro-array datasets a natural test-bed for our method. We ran our method and the $L_1$ reconstruction method on the first 500 genes from the GDS1429 dataset (http://www.ncbi.nlm.nih.gov/gds/1429), which contains 69 samples for 8565 genes. The parameters for both methods were tuned to produce a network with near to 50 edges for visualization purposes. The major connected component for each is shown in Figure \ref{fig:gene-recon}. 
\vspace{-1mm}

While these networks are too small for valid statistical analysis of the degree distribution, the submodular relaxation method produces a network with structure that is commonly seen in scale free networks. The star subgraph centered around gene 60 is more clearly defined in the submodular relaxation reconstruction, and the tight cluster of genes in the right is less clustered in the $L_1$ reconstruction. The reweighted $L_1$ method produced a quite different reconstruction, with greater clustering.
\begin{wrapfigure}{r}{6cm}
	\includegraphics[scale=0.45]{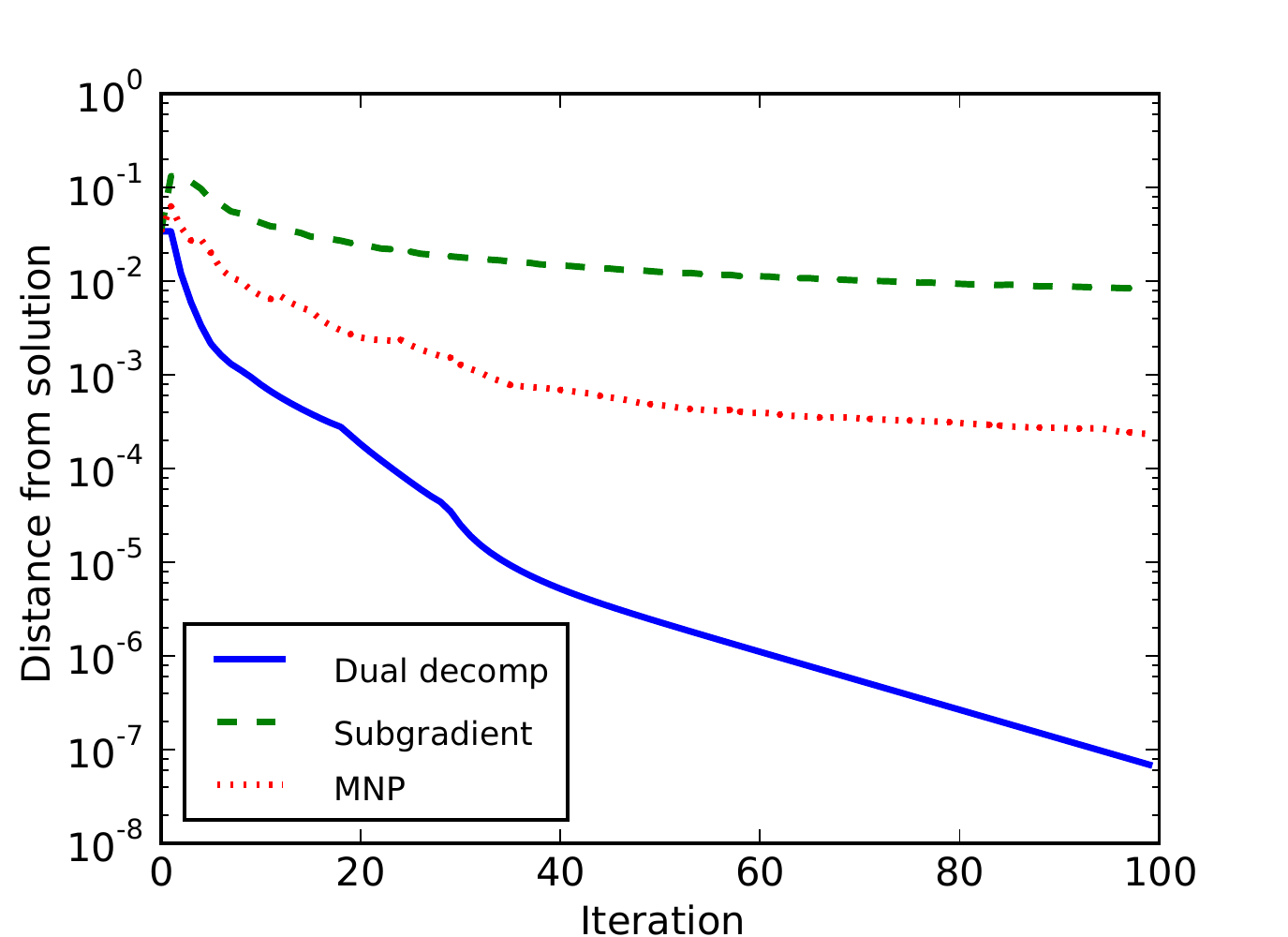}
	\caption{\small \label{fig:Comparison-of-proximal}Comparison of proximal operators}
\end{wrapfigure}
\vspace{-2mm}

\textbf{Runtime comparison: different proximal operator methods.} We performed a comparison against two other methods for computing the proximal operator: subgradient descent and the minimum norm point (MNP) algorithm. The MNP algorithm is a submodular minimization method that can be adapted for computing the proximal operator \citep{subbach}. We took the input parameters from the last invocation of the proximal operator in the BA test, at a prior strength of $0.7$. We then plotted the convergence rate of each of the methods, shown in Figure \ref{fig:Comparison-of-proximal}. As the tests are on randomly generated graphs, we present only a representative example.
It is clear from this and similar tests that we performed that the subgradient descent method converges too slowly to be of practical applicability for this
problem. Subgradient methods can be a good choice when only a low accuracy solution is required; for convergence of ADMM the error in the proximal operator needs to be smaller than what can be obtained by the subgradient method. The MNP method also converges slowly for this problem, however it achieves a low but usable accuracy quickly enough that it could be used in practice. The dual decomposition method achieves a much better rate of convergence, converging quickly enough to be of use even for strong accuracy requirements.

The time for individual iterations of each of the methods was $0.65$ms for subgradient descent, $0.82$ms for dual decomposition and $15$ms for the MNP method. The speed difference is small between a subgradient iteration and a dual decomposition iteration as both are dominated by the cost of a sort operation. The cost of a MNP iteration is dominated by two least squares solves, whose running time in the worst case is proportional to the square of the current iteration number. Overall, it is clear that our dual decomposition method is significantly more efficient.

\textbf{Runtime comparison: submodular relaxation against other approaches.} The running time of the three methods we tested is highly dependent on implementation details, so the following speed comparison should be taken as a rough guide. For a sparse reconstruction of a BA model graph with 100 vertices and 200 edges, the average running time per $10^{-4}$ error reconstruction over 10 random graphs was $16$ seconds for the reweighted $L_1$ method and $5.0$ seconds for the submodular relaxation method. This accuracy level was chosen so that the active edge set for both methods had stabilized between iterations. For comparison, the standard $L_1$ method was significantly faster, taking only $0.72$ seconds on average.
\vspace{-5mm}
\section*{Conclusion}
\label{sec:conclusion}
\vspace{-5mm}
We have presented a new prior for graph reconstruction, which enforces the recovery of scale-free networks. This prior falls within the growing class of structured sparsity methods. Unlike previous approaches to regularizing the degree distribution, our proposed prior is convex, making training tractable and convergence predictable. Our method can be directly applied in contexts where sparse covariance selection is currently used, where it may improve the reconstruction quality.
\vspace{-5mm}
\section*{Acknowledgements}
\vspace{-5mm}
NICTA is funded by the Australian Government as represented by the Department of Broadband, Communications and the Digital Economy and the Australian Research Council through the ICT Centre of Excellence program.

\small
\bibliographystyle{plainnat}
\bibliography{dp-nips2012}

\end{document}